\documentclass[12pt,draftcls,onecolumn]{IEEEtran}
\IEEEoverridecommandlockouts                              

\usepackage{amsmath,amssymb}

\DeclareMathOperator{\EX}{\mathbb{E}}

\usepackage{rotating}
\usepackage{chngcntr}
\usepackage{apptools}
\AtAppendix{\counterwithin{thm}{section}}
\usepackage{graphicx}
\usepackage{graphics}
\usepackage{amssymb}
\usepackage{amsmath}
\usepackage{color}
\usepackage{xspace}
\usepackage{bbm}
\usepackage{comment}
\usepackage{algorithm}
\usepackage{algpseudocode}
\usepackage{subfig}
\usepackage{amsthm}
\usepackage{psfrag}
\usepackage{cite}
\usepackage{tikz}
\usetikzlibrary{shapes,arrows}
\usetikzlibrary{positioning}

\tikzstyle{block}=[draw opacity=0.7,line width=1.4cm]

\definecolor{CranJ}{cmyk}{0,0.69,0.54,0.04} 
\definecolor{PinkJ}{cmyk}{0,0.71,0.43,0.12} 
\definecolor{Cran}{cmyk}{0,0.73,0.41,0.29} 
\definecolor{VRed}{cmyk}{0,0.75,0.25,0.2} 
\definecolor{ORed}{cmyk}{0,0.75,0.75,0} 
\definecolor{CBlue}{cmyk}{1,0.25,0,0} 

\title{\LARGE \bf Federated Learning Using Variance Reduced Stochastic Gradient for Probabilistically Activated Agents }

\author{Mohammadreza Rostami and Solmaz S. Kia, \emph{Senior Member, IEEE} %
  \thanks{The authors are with the Department of Mechanical and Aerospace Engineering, University of California Irvine, Irvine, CA 92697,  
    {\tt\small \{mrostam2,solmaz\}@uci.edu}. This work was supported by NSF, under CAREER Award ECCS-1653838.}%
}

\newcommand{\real}{{\mathbb{R}}} \newcommand{\reals}{{\mathbb{R}}}
 
\newcommand{\realpositive}{{\mathbb{R}}_{>0}}
\newcommand{\realnonnegative}{{\mathbb{R}}_{\ge 0}}

\newcommand{\eps}{\epsilon}

\renewcommand{\baselinestretch}{.985}

\newcommand{\vect}[1]{\boldsymbol{\mathbf{#1}}}

\allowdisplaybreaks

\newtheorem{thm}{Theorem}[section]

\newtheorem{lem}{Lemma}[section]

\newtheorem{assump}{Assumption}

\newcommand{\oprocendsymbol}{\hbox{$\bullet$}}
\newcommand{\oprocend}{\relax\ifmmode\else\unskip\hfill\fi\oprocendsymbol}


\begin{document}\fontsize{10}{11.1}\rm

\maketitle
\thispagestyle{empty}
\pagestyle{empty}

\begin{abstract}
This paper proposes an algorithm for Federated Learning (FL) with a two-layer structure that achieves both variance reduction and a faster convergence rate to an optimal solution in the setting where each agent has an arbitrary probability of selection in each iteration. In distributed machine learning, when privacy matters, FL is a functional tool. Placing FL in an environment where it has some irregular connections of agents (devices), reaching a trained model in both an economical and quick way can be a demanding job. The first layer of our algorithm corresponds to the model parameter propagation across agents done by the server. In the second layer, each agent does its local update with a stochastic and variance-reduced technique called Stochastic Variance Reduced Gradient (SVRG). We leverage the concept of variance reduction from stochastic optimization when the agents want to do their local update step to reduce the variance caused by stochastic gradient descent (SGD). We provide a convergence bound for our algorithm which improves the rate from $O(\frac{1}{\sqrt{K}})$ to $O(\frac{1}{K})$ by using a constant step-size. We demonstrate the performance of our algorithm using numerical examples.
\end{abstract}

\section{Introduction}
 In recent years, with the technological advances in modern smart devices, each phone, tablet, or smart home system, generates and stores an abundance of data, which, if harvested collaboratively with other users' data, can lead to learning models that support many intelligent applications such as smart health and image classification \cite{DCN-QVP-PNP-MD-AS-ZL-OAD-WJH:22,YZ-ML-LL-NS-DC-VC:18}. Standard traditional machine learning approaches require centralizing the training data on one machine, cloud, or in a data center. However, the data collected on modern smart devices are often of sensitive nature that discourages users from relying on centralized solutions. 
Federated Learning (FL)~\cite{QY-YL-YC-YK-TC-HY:19,HBM-EM-DR-SH-BAA:17} has been proposed to decouple the ability to do machine learning from the need to store the data in a centralized location. The idea of Federated Learning is to enable smart devices to collaboratively learn a shared prediction model while keeping all the training data on the device.

\begin{figure}[t]
\centering
    \includegraphics[scale=0.25]{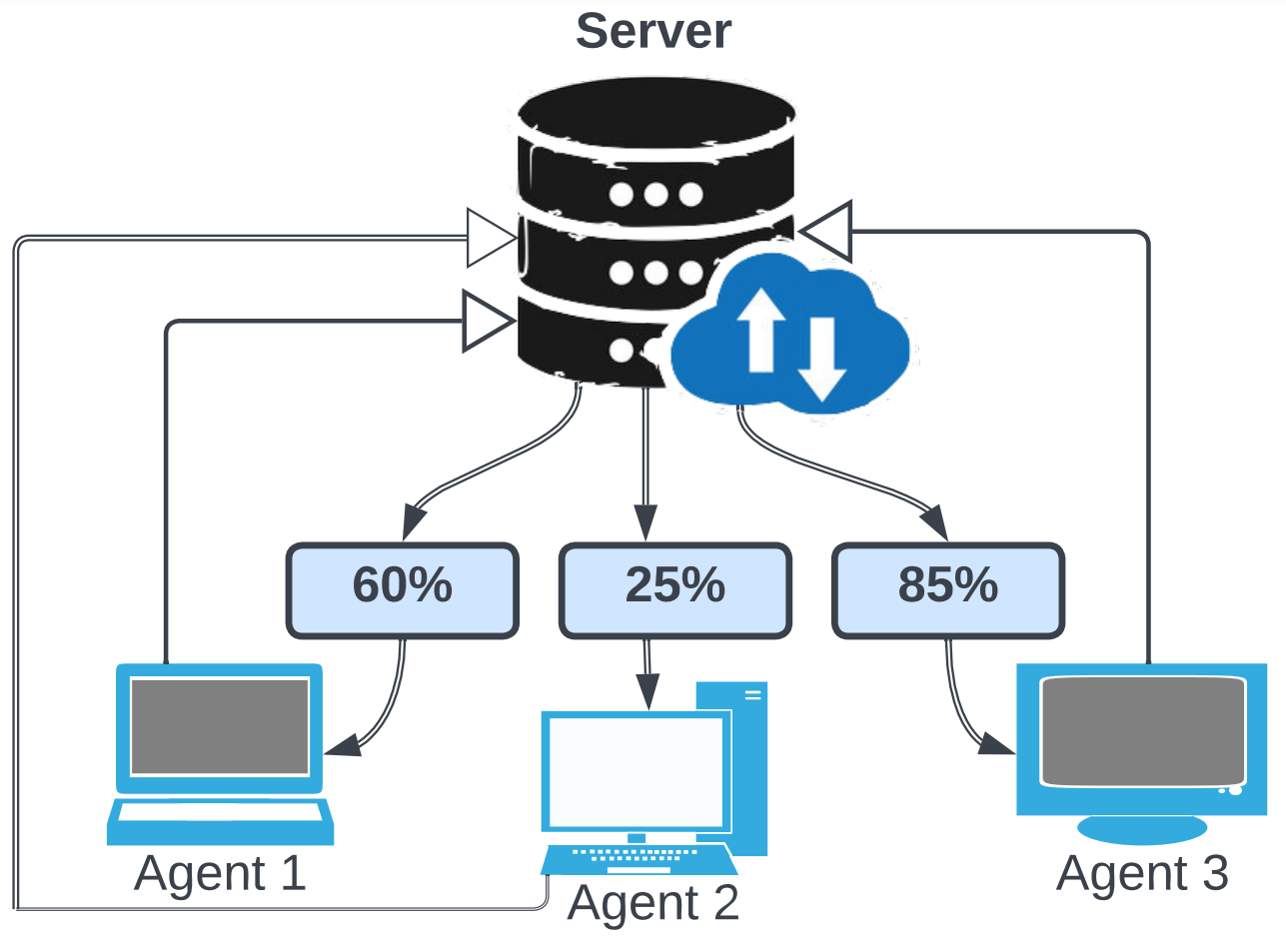}
    \caption{{\small Federated Learning structure with non-uniform probability of agent selection in each iteration}.}
    \label{fig:schematic}
\end{figure}

Figure~\ref{fig:schematic} shows a schematic representation of an FL architecture. 
In FL, collaborative learning without data sharing is accomplished by each agent receiving a current model weight from the server. Then, each participating learning separately updates the model by implementing a stochastic gradient descent (SGD) \cite{RX-SK-UAK:20} using its own locally collected datasets. Then, the participating agents send their locally calculated model weights to a server/aggregator, which often combines the models through a simple averaging, as in FedAvg \cite{HBM-EM-DR-SH-BAA:17}, to be sent back to the agents. The process repeats until a satisfactory model is obtained. Federated learning relies heavily on communication between learner agents (clients) and a moderating server. Engaging all the clients in the learning procedure at each time step of the algorithm results in huge communication cost. On the other hand, poor channel quality and intermittent connectivity can completely derail training. For resource management, in the original popular FL algorithms such as FedAvg in \cite{HBM-EM-DR-SH-BAA:17}, at each round of the algorithm, a batch of agents are selected
uniformly at random to receive the updated model weights and perform local learning. FedAvg and similar FL algorithms come with convergence guarantees~\cite{XL-KH-WY-SW-ZZ:19,AM-RJ-GJP-HH:21,TL-AKS-MZ-MS-AT-VS:20,SW-TT-TS-KKL-CM-TH-KC:19} under the assumption of availability of the randomly selected agents at each round.
However, in practice due to factors such as energy and time constraints, agents' availability is not ubiquitous at all times. Thus, some works have been done to solve this problem via device scheduling~\cite{HHY-ZL-TQQ-HVP:19,YJC-JW-GJ:20,YR-XZ-SCL-CJW:21,HY-MF-JL:21,MC-ZY-WS-CY-HVP-SC:21}. Nevertheless, the agents' availability can be a function of unforeseen factors such as communication channel quality, and thus is not deterministic and known in advance. 

To understand the effect of an agent's stochastic availability on the FL, recent work such as~\cite{JP-SW-MJ-KSC:22} proposed to move from random batch selection to an FL model where the agents availability and participation at each round are probabilistic, see Fig.~\ref{fig:schematic}. In this paper, we adopt this newly proposed framework and contribute to devising an algorithm that  achieve faster convergence and lower error covariance. Our focus will be on incorporating a variance reduction procedure into the local SGD procedure of participating learner agents at each round. The randomness in SGD algorithms induces variance of the gradient, which leads to decay learning rate and sub-linear convergence rate. Thus, there has been an effort to reduce the variance of the stochastic gradient, which resulted in the so-called \emph{Stochastic Variance Reduced Gradient} (SVRG) methods. It is shown that SVRG allows using a constant learning rate and results in linear convergence in expectation.

In this paper, we incorporate an SVRG approach in an FL algorithm where agents' participation in the update process in each round is probabilistic and non-uniform. Through rigorous analysis, we show that the proposed algorithm has a faster convergence rate. In particular, we show that our algorithm results in a practical convergence in expectation with a rate $O(\frac{1}{K})$, which is an improvement over the sublinear rate of $O(\frac{1}{\sqrt{K}})$ in~\cite{JP-SW-MJ-KSC:22}. We demonstrate the effectiveness of our proposed algorithm through a set of numerical studies and by comparing the rate of convergence, covariance, and the circular error probable (CEP) measure. Our results show that our algorithm drastically improves the convergence guarantees, thanks to the decrease in the variance, which results in faster convergence. 

\emph{Organization}: Section~\ref{sec::prelim} introduces our basic notation, 
and presents some of the properties of smooth functions. Section~\ref{se:problem-statement}
presents the problem formulation and the structure behind it.
Section \ref{section::fedavg} includes the proposed algorithm and its scheme.
Section~\ref{sec::main} contains our convergence analysis for the proposed algorithm and provides its convergence rate.
Section~\ref{sec::num} presents simulations and
Section~\ref{sec::conclu} gathers our conclusions and ideas for future
work. For the convenience of the reader, we provide some of the proofs in the Appendix.

\section{Preliminaries}\label{sec::prelim} 
In this section, we introduce our notations and terminologies used throughout the paper. We let $\reals$, $\realpositive$, $\realnonnegative$, denote the set of real, positive real numbers. Consequently, when $x\in\real$, $|x|$ is its absolute value. For $\vect{x}\in\reals^d$,
$\|\vect{x}\|=\sqrt{\vect{x}^\top\vect{x}}$ denotes the standard Euclidean norm. We let $\langle.,.\rangle$ denotes an inner product between two vectors for two vectors $x$ and $y$ $\in \real^d$.
A differentiable function $f$: $\real^d \rightarrow \real$ is Lipschitz with constant $\mathsf{L}\in\real_{>0}$, or simply $\mathsf{L}$-Lipschitz, over a set $\mathcal{C} \subseteq \real^d$ if and only if $\|f(x)-f(y)\| \leq \mathsf{L} \|x-y\|$, for $x, y \in \mathcal{C}$. Furthermore, if the function is differentiable, we have  \mbox{$f(y)\leq f(x) + \nabla f^\top(y-x)+\frac{\mathsf{L}}{2}\|y-x\|^2$} for all \mbox{$x, y\in \mathcal{C}$}~\cite{DPB:97}.
Lastly, we recall Jensen’s inequality, which states~\cite{DSB:09}:
\begin{align}
   \left \|\frac{1}{N}\sum\nolimits_{n=1}^{N}\!x_n\right\|^2 &\leq \frac{1}{N}\sum\nolimits_{n=1}^{N}\|x_n\|^2.
\end{align}

\section{Problem statement}\label{se:problem-statement}
This section formalizes the problem of interest.  Consider a set of $N$ agents (clients) that communicate with a server to learn parameters of a model that they want to fit into their collective data set. Each agent has its own local data which can be distributed either uniformly or non-uniformly. 
The learning objective is to obtain the learning function weights $\theta\in\real^d$ from 
\begin{equation}
\label{eq::sum}
    \min_{\theta\in\real^d}f(\theta) := \frac{1}{N}\sum_{n=1}^{N}f_n(\theta),\quad f_n(\theta)=\frac{1}{{L}_n}\sum_{i=1}^{L_n}f_{n,i}(\theta),
\end{equation}
where $f_n$ is possibly a convex or non-convex local learning loss function. At each agent $n\in\{1,\cdots,N\}$, $f_n$ depends on training data set $\{(q_{n,i},\hat{y}_{n,i})\}_{i=1}^{L_n}\subset \real^{1\times d}\times \real$ (supervised learning). Examples~include \cite{AG:17}
\begin{itemize}
    \item 
square loss $f_{n,i}(\theta)= \|(\hat{y}_{n,i} - q_{n,i}\theta)\|^2$,
\item log loss $f_{n,i}(\theta)=\text{log}(1+\textup{e}^{-\hat{y}_{n,i}q_{n,i} \theta})$.
\end{itemize}
\begin{assump}[Assumption on $\mathsf{L}$-smoothness of local cost functions]\label{assump::smoothness}
The local loss functions have $\mathsf{L}$-Lipschitz gradients, i.e., for any agent $n\in\{1,\cdots,N\}$ we have 
\begin{equation}
    \|\nabla f_n(\theta) - \nabla f_n(\bar{\theta})\| \leq \mathsf{L}\|\theta - \bar{\theta}\|
\end{equation}
    for any $\theta, \bar{\theta}\in\real^d$ and $\mathsf{L}\in\real_{>0}$.
\end{assump}
This assumption is technical and common in the literature.

Problem~\eqref{eq::sum} should be solved in the framework of FL in which agents maintain their local data and only interact with the server to update their local learning parameter vector based on a global feedback provided by the server. The server generates this global feedback from the local weights it receives from the agents.  
In our setting, at each round $k$ of the FL algorithm, each agent $n\in\{1,\cdots,N\}$ becomes active to perform local computations and connect to the server with a 
probability of $p_n^k$. We denote the active state by $\vect{1}_n^k\in\{0,1\}$; thus, $$p_n^k=\text{Prob}(\vect{1}_n^k=1).$$ The activation probability at each round can be different.



\section{Federated learning with variance reduction}
\label{section::fedavg}
To solve \eqref{eq::sum}, we design the FedAvg-SVRG Algorithm~\ref{alg:cap}, which has a two-layer structure. In this algorithm, each agent has its own probability to be active or passive in each round which is denoted by $p_n^k$ for agent $n$ at iteration $k$.

\begin{figure}[t]
    \centering
    \includegraphics[width=9cm, height=2cm]{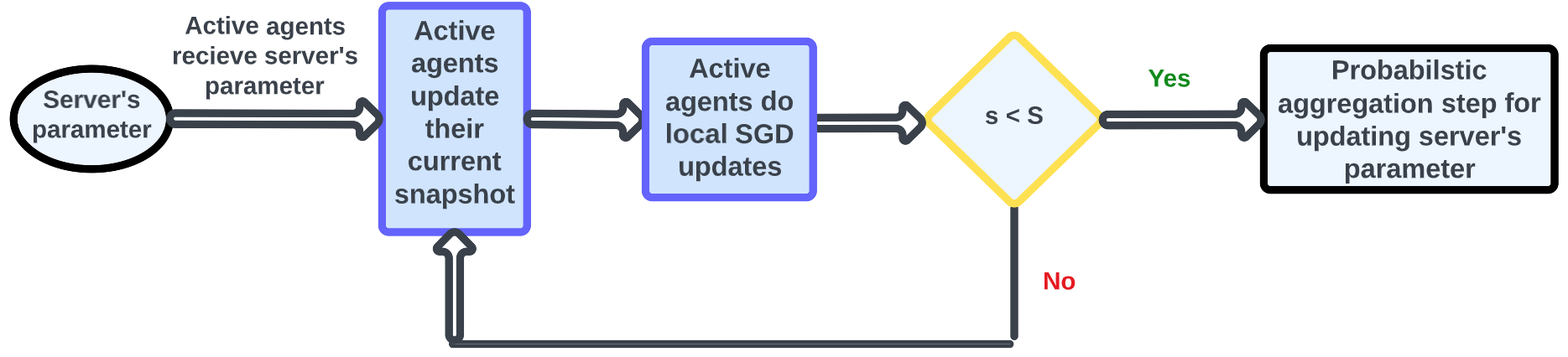}
    \caption{{\small SVRG update steps}.}
    \label{fig:svrg}
\end{figure}

Algorithm~\ref{alg:cap}  is initialized with $\theta^0$ by the server. We denote the number of the FL iterations by $K$. 
At each round $k\in\{1,\cdots,K\}$, the set of active agents is denoted by $\mathcal{A}^k$ (line 5), which is the set of agents for which $\vect{1}_n^k=1$. Then, each active agent receives a copy of the learning parameter $\theta^k$ from the server. Afterward, active agents perform their local updates according to steps 7 to 18. For resource management local update in FL algorithms, e.g.,~\cite{JP-SW-MJ-KSC:22}, follow an SGD update. However, the SGD update suffers from a high variance because of the randomized search of the algorithm, so instead of using the SGD update step, which results in a decaying step size and slow convergence, we use the SVRG update step which is stated in lines 7 to 18. In the SVRG update, we calculate the full batch gradient of the agents at some points, which are referred to as \emph{snapshots}. Then, between every two snapshots, each agent does its local update. A schematic of SVRG update steps is shown in Fig.~\ref{fig:svrg}.

We denote the number of snapshots in our SVRG method by $S$. We let $M$ be the number of local SVRG updates in between two snapshots for each active agent before aggregation. Line 10 of Algorithm~\ref{alg:cap} corresponds to computing the full batch gradient of each agent at the snapshot points, then in line 12, each agent does its local update with substituted gradient term denoted as \mbox{$v_{n,s,m}^{k} = \nabla f_\#(w_{n,s,m}^k)- \nabla f_\#(\tilde{w}) + \tilde{\mu}$}. Note the gradient substituted term in the SVRG update is an unbiased estimator. After completing the SVRG update, each agent updates its snapshot, which is mentioned in line 17\cite{RJ-TZ:13}\cite{RX-SK-UAK:20}. In the end, in line 20, the model parameter gets updated. It should be noted that the weight for updating the model parameter denoted by $\frac{\vect{1}_n^k}{p_n^k}$ makes the gradient to be unbiased when the model parameter wants to be updated because, by this fraction, agents with a low probability of being selected for each iteration still have an adequate impact on a model parameter when they play a part at each iteration. Unlike SGD, the stepsize $\delta$ for the SVRG update does not have to decay in line 14. Hence, it gives rise to a faster convergence as one can choose a large stepsize.

\begin{algorithm}[t]
\begin{algorithmic}[1]
\caption{FedAvg-SVRG with non-uniformly agent sampling}\label{alg:cap}
\label{alg:the_alg}
\State \textbf{Input:} $\delta$, $K$, $\theta^0$, $\{p_n^k\}$, $S$, $M$
\State \textbf{Output:} $\theta^k$
\For{$k \gets 0,..., K - 1$}
    \State Determine the active agents (sample $\vect{1}_n^k \sim p_n^k$)
    \State $\mathcal{A}^k\gets$\,set of active agents
    \For{$n\in\mathcal{A}^k$} 
        
         \State $ \tilde{w}_{n,0}^k = \theta^k$ 
         \For{$s \gets 0,...,S-1$}
         \State $\tilde{w}=\tilde{w}_{n,s}^k$
             \State $\tilde{\mu} = \frac{1}{L_n}\sum_{i=1}^{L_n} \nabla f_{n,i}(\tilde{w})$\vspace{0.03in}
             \State $w_{n,s,0}^k = \tilde{w}$
             \For{$m \gets 0,...,M-1$}
                \State Randomly pick $\#$  from $\in \{1,...,L_n\}$\vspace{0.03in}
                \State $w_{n,s,m+1}^k =w_{n,s,m}^k - \delta v_{n,s,m}^{k}$\vspace{0.04in}
                \State where $v_{n,s,m}^{k}= \nabla f_\#(w_{n,s,m}^k)- \nabla f_\#(\tilde{w})+\tilde{\mu}$
                \EndFor
            \State  set $\tilde{w}_{n,s+1}^k= w_{n,s,M}^k $ 
        \EndFor
    \EndFor
    \State $\theta^{k+1} = \theta^k + \frac{1}{N}\sum_{n=1}^{N} \frac{\vect{1}_n^k}{p_n^k}(w_{n,S-1,M}^k - \tilde{w}_{n,0}^k )$
    \EndFor
\end{algorithmic}
\end{algorithm}

\section{Convergence analysis }\label{sec::main}
In this section, we study the convergence bound for the proposed algorithm which is applicable for both convex and non-convex cost functions. 
\begin{assump}[Assumption on unbiased stochastic gradients]
\label{Assump:2}
\begin{equation}
    \mathbb{E} [\nabla f_\#(w)| w] = \nabla f_n(w)
\end{equation}
for any $w$ and $\# \in \{1,...,L_n\}$. \\As a result, our substituted gradient term denoted by \mbox{$v_{n,s,m}^{k} = \nabla f_\#(w_{n,s,m}^k)- \nabla f_\#(\tilde{w}) + \tilde{\mu}$} becomes unbiased where $\tilde{\mu}=\frac{1}{L_n}\sum_{i=1}^{L_n} \nabla f_{n,i}(\tilde{w})$.

\end{assump}

\begin{assump}[Bound on the substituted gradient term]
\begin{equation}\label{eq:bound}
    \EX[||v_{n,s,m}^k(w)||^2] \leq G^2, \forall \thinspace w,n,s,m \quad  
    \text{for some} \thinspace\thinspace G > 0 
\end{equation}
\end{assump}

\begin{assump}[Assumption on $\mu$-strongly convex local cost functions]\label{eq::sc}
    The local cost functions are strongly convex with parameter $\mu$, i.e.,
    \begin{equation}
        f_n(\theta_2) \geq f_n(\theta_1) + \nabla f_n(\theta_1)^T (\theta_2-\theta_1) +\frac{\mu}{2}||\theta_2 - \theta_1||^2
    \end{equation}
\end{assump}

Also, we should point out that $\vect{1}_n^k$ and $\vect{1}_{n^\prime}^k$ are independent for $n \neq n^\prime$, and the agent activation for each iteration is independent of random function selection. In other words, $\vect{1}_n^k$ and $\nabla f_\#(w)$ are completely independent.

 \begin{thm}[Convergence bound for the proposed algorithm for both convex and non-convex cost functions]\label{thm::main}
Let the assumptions \ref{assump::smoothness} and \ref{Assump:2} hold. Then Algorithm~\ref{alg:cap} results in 
\begin{align}\label{thm1}
    &\frac{1}{K}\sum_{k=0}^{K-1}\EX\bigg[ \bigg \|\nabla f(\theta^k) \bigg\|^2 \nonumber\bigg]\leq \frac{2}{\delta K S M}(f(\theta^0) -f^\star)\\
&+\frac{\delta^2 \mathsf{L}^2 (M-1)}{K S M N} \nonumber\\ 
&\Bigg[ \sum_{k=0}^{K-1}\sum_{n=1}^{N}\sum_{s=0}^{S-1}\sum_{m=0}^{M-1}\sum_{m^\prime=0}^{m-1}\EX\bigg[\bigg\|v_{n,s,m^\prime}^k \bigg\|^2\bigg] \Bigg] \nonumber \\
&+\frac{\delta \mathsf{L}}{KN}\sum_{k=0}^{K-1}\sum_{n=1}^{N}\frac{1}{p_n^k}\sum_{s=0}^{S-1}\sum_{m=0}^{M-1}\EX\bigg[\bigg\|v_{n,s,m}^{k}\bigg\|^2\bigg].
\end{align}
\\
Furthermore, if assumption (\ref{eq:bound}) holds we can write the following bound on the right hand side:
\begin{align}\label{eq::bound}
    &\frac{1}{K}\sum_{k=0}^{K-1}\EX\bigg[ \bigg \|\nabla f(\theta^k) \bigg\|^2 \nonumber\bigg]\leq \frac{2}{\delta K S M}(f(\theta^0) -f^\star)\\
&+ {\delta^2 \mathsf{L}^2 (M-1)^2}G^2 \nonumber\\
&+\frac{\delta \mathsf{L} SMG^2}{KN}\sum_{k=0}^{K-1}\sum_{n=1}^{N}\frac{1}{p_n^k},
\end{align}
where $f^\star$ is the optimal solution to \eqref{eq::sum}.
\end{thm}

Proof of Theorem~\ref{thm::main} is given in the appendix.

\textbf{\emph{Remark}}: According to (\ref{eq::bound}) a rate of convergence of the algorithm is determined by $\min[\frac{1}{\delta K},\delta^2,\frac{\delta}{K}]$. In order to select the convergence rate of our algorithm we can derive it by choosing $\delta = \frac{1}{K^\epsilon}$. Then, the rate of convergence is chosen from $\min[\frac{1}{ K^{1-\epsilon}},\frac{1}{K^{2\epsilon}},\frac{1}{K^{1+\epsilon}}]$. By selecting $\epsilon = \frac{1}{3}=\max\{1-\eps,2\eps\}$ the best convergence rate can be obtained, which is of order $O(\frac{1}{\sqrt[3]{K^2}})$. Thus, using the decaying step-size ($\delta = \frac{1}{\sqrt[3]{K}}$) allows us obtain $O(\frac{1}{\sqrt[3]{K^2}})$ convergence to the optimal point for both convex and non-convex cost functions.

\begin{lem}
    If assumption (\ref{eq::sc}) holds, then algorithm \ref{alg:the_alg} will convergence to the optimal point with rate no less than $O(\frac{1}{K})$.\\
    \begin{proof}
        Without loss of generality, let say $S = 1$, \emph{i.e,} no intermediate full gradient calculation for each agent. Then, we have the following upper bound for the substituted gradient term:\\
        
        \begin{equation*}
            v_{n,m^\prime}^{k}= \nabla f_\#(w_{n,m^\prime}^k)- \nabla f_\#(\tilde{w})+\tilde{\mu}
        \end{equation*}
        Taking expectation with respect to $\#$, using the facts that \mbox{$||x+y||_2^2 \leq 2||x||_2^2 + 2||y||_2^2$}, $\tilde{\mu} = \nabla f_n(\tilde{w})$ and $\EX||x-\EX x||_2^2 \leq  \EX||x||_2^2$, we get the following inequalities:
        \begin{align*}
            \EX||v_{n,m^\prime}^{k}||_2^2 &\leq 2\EX||\nabla f_\#(w_{n,m^\prime}^k) - \nabla f_\#(w^\star_n) ||_2^2\\
            &+ 2\EX||[\nabla f_\#(\tilde{w}) - \nabla f_\#(w^\star_n)] - \nabla f_n(\tilde{w})||_2^2\\
            &=2\EX||\nabla f_\#(w_{n,m^\prime}^k) - \nabla f_\#(w^\star_n) ||_2^2\\
            &+ 2\EX||[\nabla f_\#(\tilde{w}) - \nabla f_\#(w^\star_n)]\\
            &- \EX[\nabla f_\#(\tilde{w})-\nabla f_\#(w^\star_n)] ||_2^2\\
            &\leq 2\EX||\nabla f_\#(w_{n,m^\prime}^k) - \nabla f_\#(w^\star_n) ||_2^2\\
            &+ 2\EX||[\nabla f_\#(\tilde{w}) - \nabla f_\#(w^\star_n)]||_2^2\\
            &\leq 4\mathsf{L} [f_n(w_{n,m^\prime}^k) -f_n(w^\star_n)+f_n(\tilde{w})-f_n(w^\star_n)].
        \end{align*}
        Where in the last inequality we use theorem 1 in \cite{RJ-TZ:13}.\\
        By summing both sides for $m^\prime =0 : m-1$ and selecting $w_{n,m}^k = w_{n,m^\prime}^k$ for randomly chosen $m^\prime \in \{0,...,m-2\}$ which is a valid update scheme in SVRG, we get the following:
        \begin{align}\label{expec:bound}
            \sum_{m^\prime = 0}^{m-1} \EX||v_{n,m^\prime}^{k}||_2^2 &\leq \nonumber \sum_{m^\prime = 0}^{m-1} 4\mathsf{L} [f_n(w_{n,m^\prime}^k) \nonumber\\
            &-f_n(w^\star_n)+f_n(\tilde{w})-f_n(w^\star_n)]\nonumber\\
            & =4\mathsf{L} m \EX[f_n(w_{n,m}^k) \nonumber\\
            &-f_n(w^\star_n)+f_n(\tilde{w})-f_n(w^\star_n)]\nonumber\\
            &= 4\mathsf{L} m \bigg[\EX[f_n(w_{n,m}^k)-f_n(w^\star_n)]\nonumber\\
            &+\EX[f_n(\tilde{w})-f_n(w^\star_n)]\bigg],
        \end{align}
        where we use the fact that $\EX[x+y] = \EX[x]+\EX[y]$ in the last equality.
            Next, invoking~\cite[Theorem 1]{RJ-TZ:13}, we can write
            \begin{equation}\label{eq:bound:svrg}
                \EX[f_n(w_{n,m}^k) - f_n(w^\star_n)] \leq \alpha^{k}\EX[f_n(\tilde{w}_{n}^k) - f_n(w^\star_n)],
            \end{equation}
            \begin{equation}
                 \EX[f_n(w_{n,m}^k) - f_n(w^\star_n)] \leq \alpha^{k}D,
            \end{equation}
            where $\alpha = h(\mu,\mathsf{L},\delta,M) < 1$ and $D$ is an upper bound of $\EX[f_n(\tilde{w}_{n}^k) - f_n(w^\star_n)]$, then we have geometric convergence in expectation for SVRG.


        Using (\ref{eq:bound:svrg}) we can upper bound the last two expectations in (\ref{expec:bound}) and get the following result:

        \begin{align}\label{eq::final bound}
            \sum_{m^\prime = 0}^{m-1} \EX||v_{n,m^\prime}^{k}||_2^2 &= 4\mathsf{L} m \bigg[\EX[f(\tilde{w}_{n}^k)-f(w^\star_n)]\nonumber\\
            &+\EX[f(\tilde{w})-f(w^\star_n)]\bigg]\nonumber\\
            &\leq 8\mathsf{L} m\thinspace \alpha^{k}\EX[f_n(\tilde{w}_{n}^k) - f_n(w^\star_n)]\nonumber\\
            &\leq 8\mathsf{L} m\thinspace \alpha^{k} D,
        \end{align}
        note that because $\alpha<1$, we have:
        \begin{equation}\label{seris}
            \sum_{k=0}^{K-1}\alpha^k = \frac{1-\alpha^K}{1-\alpha}.
        \end{equation} 
        by using (\ref{eq::final bound}) and (\ref{seris}) into (\ref{thm1}) we can establish $O(\frac{1}{K})$ convergence rate for the proposed algorithm.
        \begin{align}
    &\frac{1}{K}\sum_{k=0}^{K-1}\EX\bigg[ \bigg \|\nabla f(\theta^k) \bigg\|^2 \nonumber\bigg]\leq \frac{2}{\delta K M}(f(\theta^0) -f^\star)\\
    &+\frac{8D\delta^2 \mathsf{L}^3 (M-1)^2 (1-\alpha^K)}{K (1-\alpha)}  \nonumber \\
    &+\frac{8D\delta \mathsf{L^2}(M-1)}{KN}\sum_{k=0}^{K-1}\alpha^{k}\sum_{n=1}^{N}\frac{1}{p_n^k}.
    \end{align}
    \end{proof}
\end{lem}
By incorporating a SVRG approach in our FL algorithm, Theorem~\ref{thm::main} guarantees that we can a fixed size step-size $\delta$ and achieve a convergence rate of $O(\frac{1}{K})$. The improvement is due to the fact that the SVRG update step does not need to have a decaying stepsize throughout the learning process. Thus, using a constant and larger stepsize leads the algorithm to faster convergence to the stationary point.
This is an improvement over the existing algorithm~\cite{JP-SW-MJ-KSC:22} in which they guarantee $O(\frac{1}{\sqrt{K}})$ as the convergence rate of the algorithm by using the SGD method for their local update step.

\section{Numerical Simulations}\label{sec::num}
In this section, we analyze and demonstrate the performance of the Algorithm~\ref{alg:the_alg} by solving a  regression problem (quadratic loss function). In this study, we compare the performance of our algorithm to that of the FedAvg in~\cite{JP-SW-MJ-KSC:22}. We used a real medical insurance data set of $900$ persons in the form of $(y,v)\in\real\times\real^{1\times 5}$. Then, to observe the effect of stochastic optimization, we distribute the data among $18$ agents. Thus, each agent owns $50$ quadratic costs. In other words, we seek to solve the following convex optimization problem:
\begin{align}
        &\min_{\theta\in\real^{10}} f(\theta)=\frac{1}{N}\sum_{n=1}^{N}f_n(\theta),
        \\ & ~f_n(\theta)=\frac{1}{{L}_n}\sum_{i=1}^{L_n}f_{n,i}(\theta),~~f_{n,i}(\theta)= \|q_{n,i}\theta-\hat{y}_{n,i} \|^2,\nonumber
\end{align}

\begin{figure}[t]
    \centering
    \includegraphics[scale=0.2]{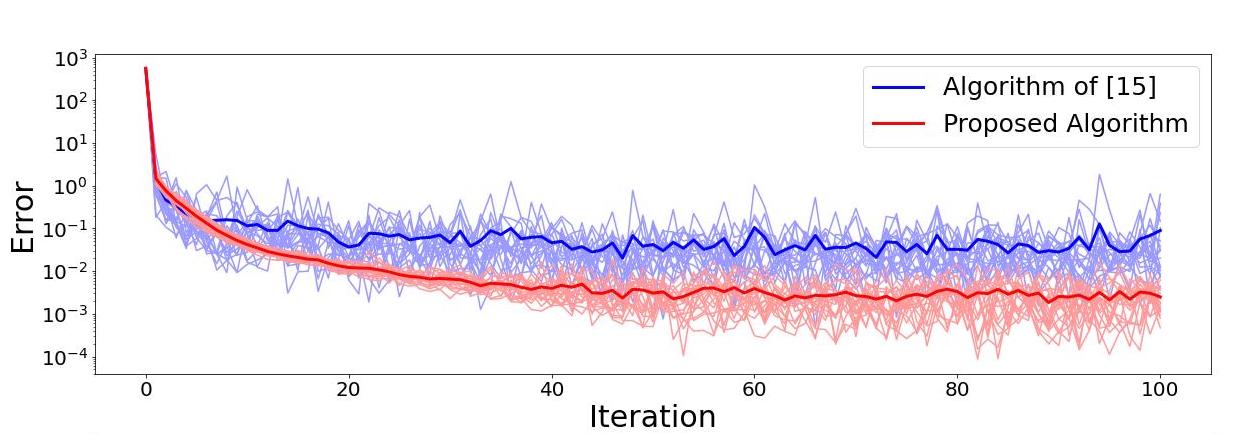}
    \caption{{\small Error of cost function for the first case (\mbox{$S=5$ and $M=2$}) over 20 Monte Carlo iterations (thicker line correspond to the mean of Monte Carlo iterations, vertical axis is limited for the purpose of better visualization).}}
    \label{fig::cost_1}
\end{figure}

\begin{figure}[t]
    \centering
    \includegraphics[scale=0.2]{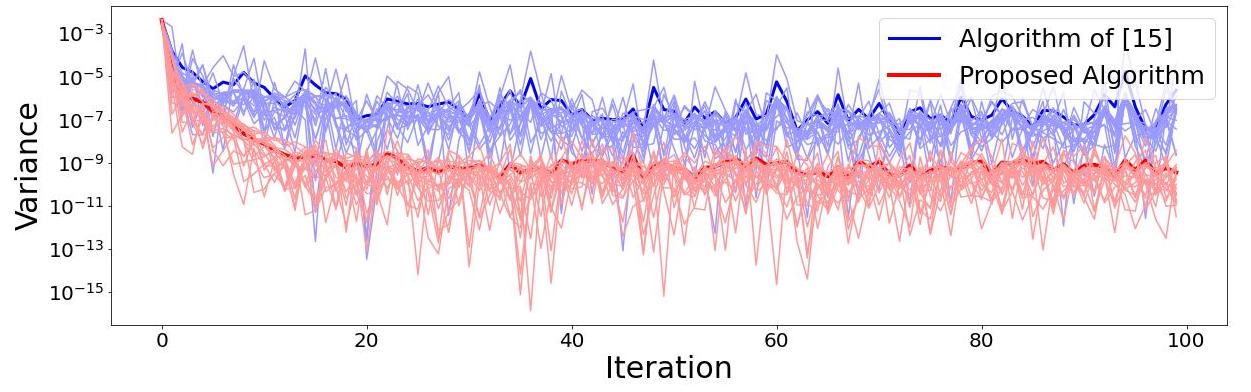}
    \caption{{\small Variance of cost function for the first case (\mbox{$S=5$ and $M=2$}) over 20 Monte Carlo iterations (thicker line correspond to the mean of Monte Carlo iterations, vertical axis is limited for the purpose of better visualization)}.}
    \label{fig::variance_1}
\end{figure}
where in our problem, $N$, and $L_n$ are 18 and 50, respectively. Here, $\hat{y}_{n,i} \in \real$, and $\theta$ is the learning parameter (weight) which is a column vector with 5 elements.
We conduct $20$ Monte Carlo simulation in all of which we initialize our algorithm at  $\theta^0= [0.5,...,0.5]^\top$, and we use the fixed  step-size  $\delta = \frac{1}{\sqrt{100}}$ in all rounds.
We also simulate the FedAvg algorithm of~\cite{JP-SW-MJ-KSC:22} with the same initialization but using the decaying  stepsize of $\frac{1}{\sqrt{K}}$ as mentioned in~\cite{JP-SW-MJ-KSC:22}. For our algorithm we consider two cases: (1) $(S=5,M=2)$ and (2) $(S=10,M=5)$.

\begin{figure}[t]
    \centering
    \includegraphics[scale=0.27]{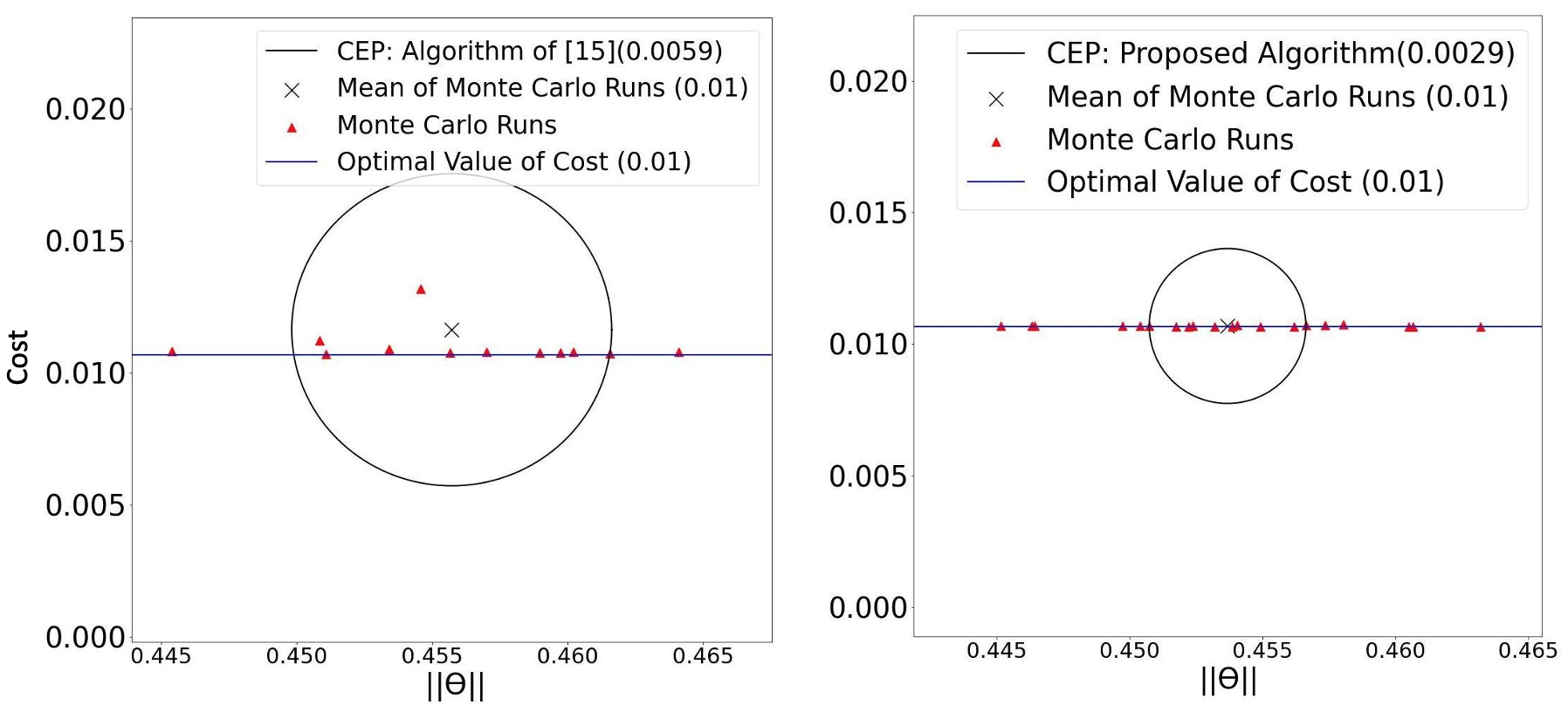}
    \caption{{\small CEP graph for the second case (\mbox{$S=5$ and $M=2$}) over 20 Monte Carlo iterations. The results for our algorithm is in the right while the result for algorithm of [15] is shown on the left}.}  
    \label{fig::CEP_1}
\end{figure}

\begin{figure}[t]
    \centering
    \includegraphics[scale=0.2]{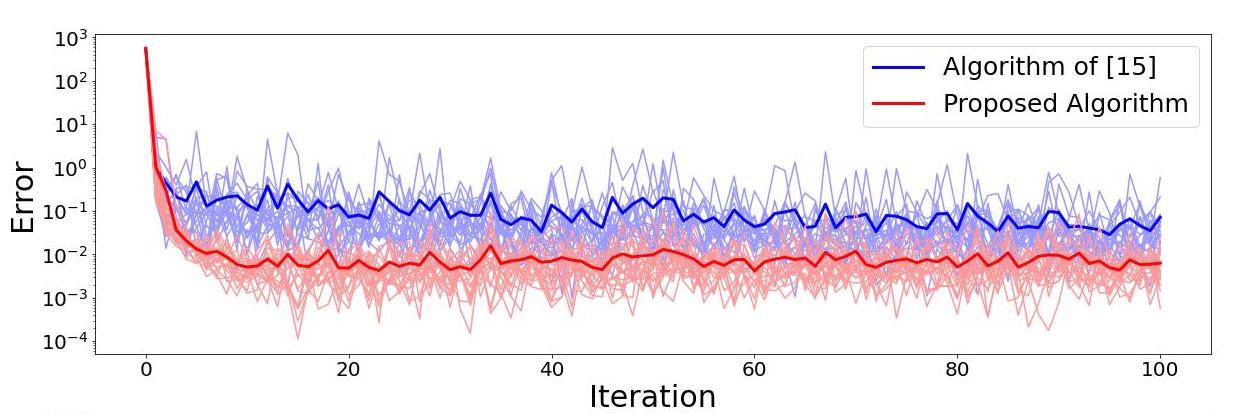}
    \caption{{\small Error of cost function for the second case (\mbox{$S=10$ and $M=5$}) over 20 Monte Carlo iterations (thicker line correspond to the mean of Monte Carlo iterations, $y$ axis is limited for the purpose of better visualization)}.}
    \label{fig::cost_2}
\end{figure}

The simulation results for the first case are shown in Fig.~\ref{fig::cost_1}--Fig.\ref{fig::CEP_1}, while the results for the second case are shown in Fig.~\ref{fig::cost_2}--Fig.\ref{fig::CEP_2}. Figures \ref{fig::cost_1} and \ref{fig::cost_2} show that in both cases our algorithm has a faster convergence to the optimal cost (the value is $0.01$). 

Figures \ref{fig::variance_1} and \ref{fig::variance_2} show the variance caused by the two algorithms. The variance of our algorithm is significantly lower than that of the algorithm of~\cite{JP-SW-MJ-KSC:22}.  In order to show the variance of our algorithm, we put a logarithmic axis on y axis. Also, the variance of our algorithm decreases as the number of iterations increases as opposed to the algorithm of~\cite{JP-SW-MJ-KSC:22}, which suffers from a high variance.

Figure \ref{fig::CEP_1} and \ref{fig::CEP_2} show the circular error probable (CEP) to observe the variance in the last iteration \mbox{($K=100$)} for our algorithm and the FedAvg algorithm in \cite{JP-SW-MJ-KSC:22}. CEP is a measure used in navigation filters. It is defined as the radius of a circle, centered on the mean, whose perimeter is expected to include the landing points of 50\% of the rounds; said otherwise, it is the median error radius~\cite{JCS-JLM:92}. 
Here, then, CEP demonstrates how far the means of the Monte Carlo runs are from 50\% of the Monte Carlo iterations for both algorithms. As a result, less radius means less variance from the mean of the Monte Carlo runs. This plot shows not only our algorithm reaches a closer neighborhood to the optimal cost, but also, it has less CEP radius in comparison to that of the algorithm of \cite{JP-SW-MJ-KSC:22}; this is another indication that our algorithm has less variance compared to the FedAvg algorithm in \cite{JP-SW-MJ-KSC:22}. For our algorithm, the CEP radius in the first and the second cases are respectively $0.0029$ and $0.0077$, while these values of the algorithm of \cite{JP-SW-MJ-KSC:22} are respectively $0.0059$ and $0.0201$.

To complete our simulation study, we also compare the convergence performance of our algorithm to that of the FedAvg of \cite{HBM-EM-DR-SH-BAA:17}, which uses a 
uniform agent selection. Figure~\ref{fig::all_models} demonstrates the results when we use the batch size of $5$ of the FedAvg of \cite{HBM-EM-DR-SH-BAA:17} and use the parameters corresponding to the first case for our algorithm. As we can see, our algorithm outperforms the FedAvg of~\cite{HBM-EM-DR-SH-BAA:17} both in mean and variance.

\begin{figure}[t]
    \centering
    \includegraphics[scale=0.2]{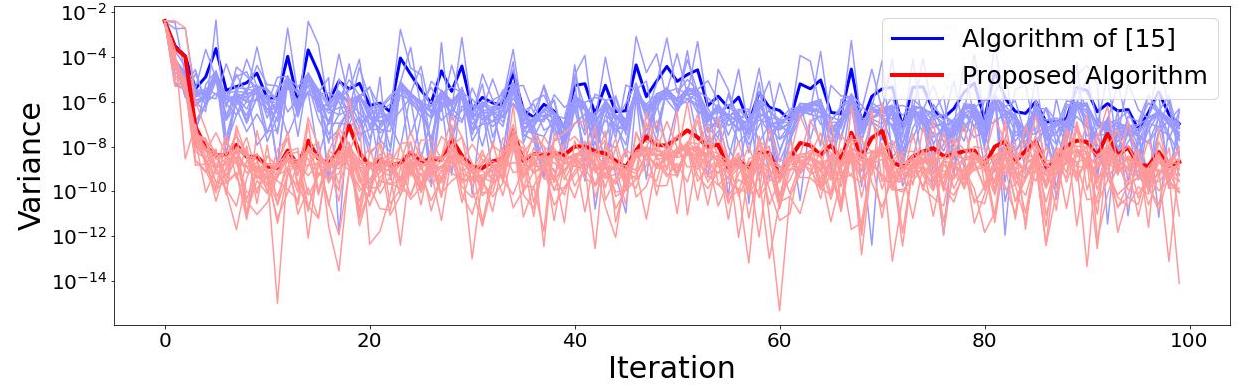}
    \caption{{\small Variance of cost function for the second case (\mbox{$S=10$ and $M=5$}) over 20 Monte Carlo iterations (thicker line correspond to the mean of Monte Carlo iterations, $y$ axis is limited for the purpose of better visualization).}}
    \label{fig::variance_2}
\end{figure}



\begin{figure}[t!]
    \centering
    \includegraphics[scale=0.35]{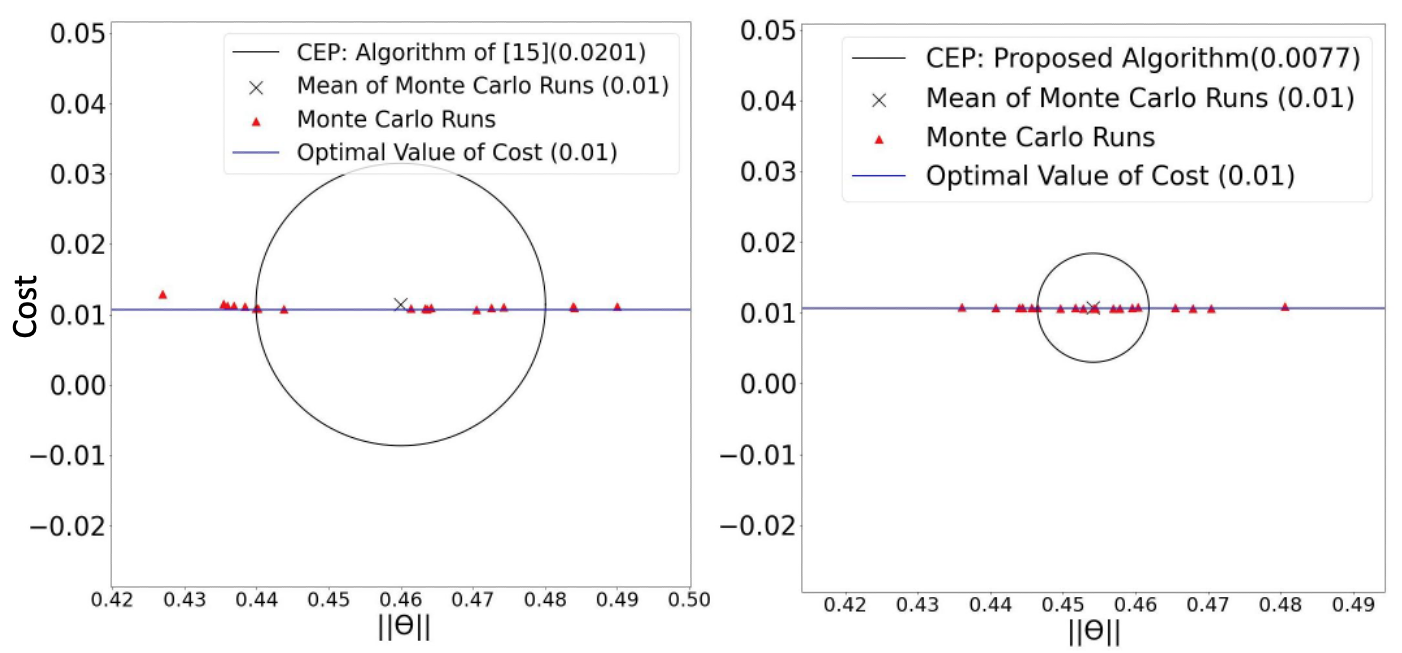}
\caption{{\small CEP graph for the second case (\mbox{$S=10$ and $M=5$}) over $20$ Monte Carlo iterations. The results for our algorithm is in the right while the result for algorithm of [15] is shown on the left.}}  
\label{fig::CEP_2}
\end{figure}

\begin{figure}[t]
    \centering  
    \includegraphics[scale=0.2]{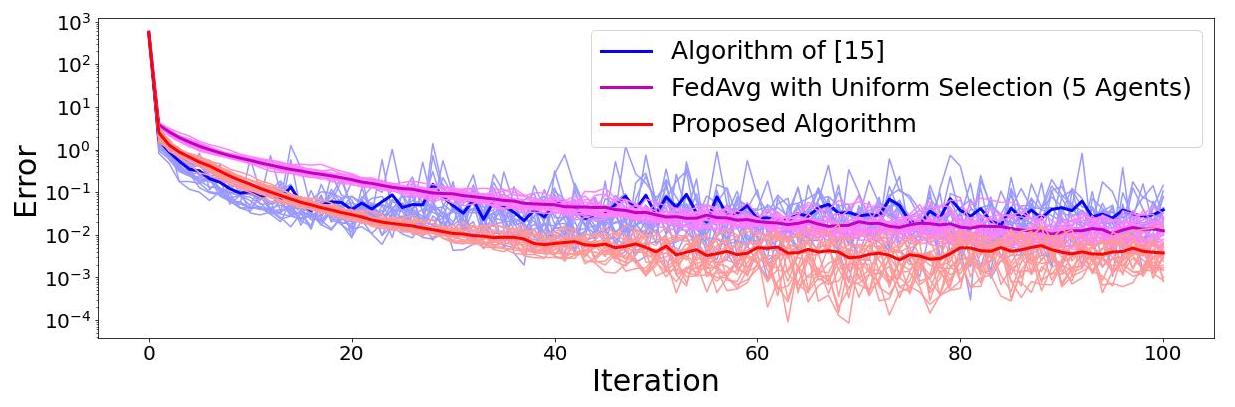}
    \caption{{\small Comparing the results of the cost function for three algorithms over 20 Monte Carlo iterations (thicker line correspond to the mean of Monte Carlo iterations, $y$ axis is limited for the purpose of better visualization).}}
    \label{fig::all_models}
\end{figure}

\section{Conclusions}\label{sec::conclu}
We have proposed an algorithm in the FL framework in the setting where each agent can have a non-uniform probability of becoming active (getting selected) in each FL round. The algorithm possesses a doubly-layered structure as the original FL algorithms. The first layer corresponds to distributing the server parameter to the agents. At the second layer, each agent updates its copy of the server parameter through an SVRG update. Then after each agent sends back its update, the server parameter gets updated. By leveraging the SVRG technique from stochastic optimization, we constructed a local updating rule that allowed the agents to use fixed stepsize.  We characterized an upper bound for the gradient of the expected value of the cost function, which showed our algorithm converges to the optimal solution with the rate of no less than $O(\frac{1}{K})$. This showed an improvement over the existing results that only have a convergence rate of  $O(\frac{1}{\sqrt{K}})$.  We demonstrated the performance of our algorithm through several numerical examples. We used various statistical measures to show our algorithm's faster convergence and low variance compared to some state-of-the-art existing FL algorithms.  Future work will investigate the extension of the result for the non-uniformly selection of snapshots inside the SVRG update for computing the full batch gradient of the agents.


\bibliographystyle{ieeetr}%
\bibliography{bib/alias,bib/Reference} 

\vspace{-0.1in}
\appendix
\renewcommand{\theequation}{A.\arabic{equation}}
\renewcommand{\thethm}{A.\arabic{thm}}
\renewcommand{\thelem}{A.\arabic{lem}}
\renewcommand{\thedefn}{A.\arabic{defn}}
This appendix gives the proof of Theorem~\ref{thm::main}. However, before giving that proof we state some auxiliary lemmas that we will invoke in our proof. 

\begin{lem} \label{lem::1}
Consider Algorithm~\ref{alg:cap}. We can establish that 
\begin{align}\label{eq::upperbound}
      \frac{\delta^2\mathsf{L}}{2N^2}\EX\bigg[\bigg\|\sum_{n=1}^{N}\frac{\vect{1}_n^k}{p_n^k}\sum_{s=0}^{S-1}\sum_{m=0}^{M-1}v_{n,s,m}^{k}\bigg\|^2\bigg| \theta^k\bigg] \leq NSM\sum_{n=1}^{N}\frac{1}{p_n^k}\sum_{s=0}^{S-1}\sum_{m=0}^{M-1}\EX\bigg[\bigg\|v_{n,s,m}^{k}\bigg\|^2\bigg| \theta^k\bigg]  
\end{align}
\end{lem}
\begin{proof}
By using Jensen's inequality we get the following inequalities:
\begin{align*}
&\EX\bigg[\bigg\|\sum_{n=1}^{N}\frac{\vect{1}_n^k}{p_n^k}\sum_{s=0}^{S-1}\sum_{m=0}^{M-1}v_{n,s,m}^{k}\bigg\|^2\bigg| \theta^k\bigg]\\
&\leq N\sum_{n=1}^{N}\EX\bigg[\bigg\|\frac{\vect{1}_n^k}{p_n^k}\sum_{s=0}^{S-1}\sum_{m=0}^{M-1}v_{n,s,m}^{k}\bigg\|^2\bigg| \theta^k\bigg]\\
&=N\sum_{n=1}^{N}\frac{\EX[\vect{1}_n^k|\theta^k]}{(p_n^k)^2}\EX\bigg[\bigg\|\sum_{s=0}^{S-1}\sum_{m=0}^{M-1}v_{n,s,m}^{k}\bigg\|^2\bigg| \theta^k\bigg]\\
&\leq NS\sum_{n=1}^{N}\frac{\EX[\vect{1}_n^k|\theta^k]}{(p_n^k)^2}\sum_{s=0}^{S-1}\EX\bigg[\bigg\|\sum_{m=0}^{M-1}v_{n,s,m}^{k}\bigg\|^2\bigg| \theta^k\bigg]\\
   &\leq NSM\sum_{n=1}^{N}\frac{\EX[\vect{1}_n^k|\theta^k]}{(p_n^k)^2}\sum_{s=0}^{S-1}\sum_{m=0}^{M-1}\EX\bigg[\bigg\|v_{n,s,m}^{k}\bigg\|^2\bigg| \theta^k\bigg]\\
&\leq NSM\sum_{n=1}^{N}\frac{1}{p_n^k}\sum_{s=0}^{S-1}\sum_{m=0}^{M-1}\EX\bigg[\bigg\|v_{n,s,m}^{k}\bigg\|^2\bigg| \theta^k\bigg] 
\end{align*}
Which concludes the proof.
\end{proof}

\begin{lem}\label{lemm 4.2}
Consider Algorithm~\ref{alg:cap}. We can establish that 
\begin{multline}\label{eq::12}\\
-\delta\EX\bigg[\bigg\langle \nabla f(\theta^k), \frac{1}{N}\sum_{n=1}^{N}\nabla f_n(w_{n,s,m}^k)\bigg\rangle\bigg]
\leq\frac{\delta^3 \mathsf{L}^2(M-1)}{2N}\sum_{n=1}^{N}\sum_{m^\prime=0}^{m-1}\EX\bigg[\bigg\| v_{n,s^,m^\prime}^k \bigg\|^2\bigg] -\frac{\delta}{2}\EX\bigg[\bigg\|\nabla f(\theta^k)\bigg\|^2\bigg]
\end{multline}
\end{lem}
\begin{proof} 
We start by noting that 
\begin{align*}
&-\delta\EX\bigg[\bigg\langle \nabla f(\theta^k), \frac{1}{N}\sum_{n=1}^{N}\nabla f_n(w_{n,s,m}^k)\bigg\rangle\bigg]=-\delta\EX\bigg[\bigg\langle \nabla f(\theta^k), \frac{1}{N}\sum_{n=1}^{N}\nabla f_n(w_{n,s,m}^k)-\nabla f(\theta^k)+\nabla f(\theta^k) \bigg\rangle\bigg]\\
&=\delta\EX\bigg[\bigg\langle \nabla f(\theta^k), -\frac{1}{N}\sum_{n=1}^{N}\nabla f_n(w_{n,s,m}^k) +\nabla f(\theta^k) \bigg\rangle\bigg]-\delta\EX\bigg[\bigg\langle\nabla f(\theta^k), \nabla f(\theta^k)\bigg\rangle\bigg]\\
&\leq \frac{\delta}{2}\EX\bigg[\bigg\|\nabla f(\theta^k)\bigg\|^2\bigg] - \delta\EX\bigg[\bigg\|\nabla f(\theta^k)\bigg\|^2\bigg]+\frac{\delta}{2}\EX\bigg[\bigg\|\nabla f(\theta^k)- \frac{1}{N}\sum_{n=1}^{N}\nabla f_n(w_{n,s,m}^k)\bigg\|^2\bigg] \\
&=\frac{\delta}{2}\EX\bigg[\bigg\|\nabla f(\theta^k)\bigg\|^2\bigg] - \delta\EX\bigg[\bigg\|\nabla f(\theta^k)\bigg\|^2\bigg] + \frac{\delta}{2}\EX\bigg[\bigg\|\frac{1}{N}\sum_{n=1}^{N}\big[\nabla f_n(\theta^k)- \nabla f_n(w_{n,s,m}^k)\big]\bigg\|^2\bigg]\\
&\leq\frac{\delta}{2N}\sum_{n=1}^{N}\EX\bigg[\bigg\|\nabla f_n(\theta^k)- \nabla f_n(w_{n,s,m}^k)\bigg\|^2\bigg]-\frac{\delta}{2}\EX\bigg[\bigg\|\nabla f(\theta^k)\bigg\|^2\bigg]\\
&\leq\frac{\delta \mathsf{L}^2}{2N}\sum_{n=1}^{N}\EX\bigg[\bigg\|\theta^k - w_{n,s,m}^k\bigg\|^2\bigg] -\frac{\delta}{2}\EX\bigg[\bigg\|\nabla f(\theta^k)\bigg\|^2\bigg]\\
&=\frac{\delta \mathsf{L}^2}{2N}\sum_{n=1}^{N}\EX\bigg[\bigg\|\sum_{m^\prime=0}^{m-1}\delta v_{n,s,m^\prime}^k \bigg\|^2\bigg] -\frac{\delta}{2}\EX\bigg[\bigg\|\nabla f(\theta^k)\bigg\|^2\bigg]\\
&\leq \frac{\delta^3 \mathsf{L}^2}{2N}\sum_{n=1}^{N}\EX\bigg[\bigg\|\sum_{m^\prime=0}^{m-1} v_{n,s,m^\prime}^k \bigg\|^2\bigg] -\frac{\delta}{2}\EX\bigg[\bigg\|\nabla f(\theta^k)\bigg\|^2\bigg]\\
&\leq \frac{\delta^3 \mathsf{L}^2(M-1)}{2N}\sum_{n=1}^{N}\sum_{m^\prime=0}^{m-1}\EX\bigg[\bigg\| v_{n,s,m^\prime}^k \bigg\|^2\bigg]-\frac{\delta}{2}\EX\left[\Big\|\nabla f(\theta^k)\Big\|^2\right]
\end{align*}
Which concludes the proof. Where in the third inequality, we use the smoothness property of functions. For rest of inequalities we use Jensen's inequality.
\end{proof}
We are now ready to present the proof of Theorem~\ref{thm::main}.

\medskip
\begin{proof}[Proof of Theorem\ref{thm::main}]
Our proof is based on the smoothness of the cost function.
Then, from the smoothness, we can write the following inequality:
\begin{align}\label{eq::smooth}
    \EX[f(\theta^{k+1})|\theta^k] \leq f(\theta^k)&+\langle\nabla f(\theta^k), \EX[\theta^{k+1} - \theta^k|\theta^k] \rangle \nonumber \\
    &+\frac{\mathsf{L}}{2}\EX\big[\|\theta^{k+1} - \theta^k\|^2\big| \theta^k\big] 
\end{align}

From the proposed algorithm, we can write the following:
\begin{align}\label{eq::1}
\theta^{k+1} - \theta^k  &= \frac{1}{N}\sum_{n=1}^{N} \frac{\vect{1}_n^k}{p_n^k}(w_{n,S-1,M}^k -\tilde{w}_{n,0}^k ) \nonumber\\
 &=-\frac{\delta}{N}\sum_{n=1}^{N}\frac{\vect{1}_n^k}{p_n^k}\sum_{s=0}^{S-1}\sum_{m=0}^{M-1}v_{n,s,m}^{k}
\end{align}

From \eqref{eq::smooth} and \eqref{eq::1}, we have the following:
\begin{multline}\label{eq::eqq}
 \EX[f(\theta^{k+1})|\theta^k] \leq f(\theta^k)+\langle\nabla f(\theta^k), \EX[\theta^{k+1} - \theta^k|\theta^k] \rangle 
    +\frac{\mathsf{L}}{2}\EX\big[\|\theta^{k+1} - \theta^k\|^2\big| \theta^k\big]  \\
    =f(\theta^k) -\delta \bigg\langle \nabla f(\theta^k), \EX\bigg[\frac{1}{N}\sum_{n=1}^{N}\frac{\vect{1}_n^k}{p_n^k}\sum_{s=0}^{S-1}\sum_{m=0}^{M-1}v_{n,s,m}^{k}\bigg| \theta^k\bigg]\! \bigg\rangle\\
    +\frac{\delta^2 \mathsf{L}}{2N^2}\EX\bigg[\bigg\|\sum_{n=1}^{N}\frac{\vect{1}_n^k}{p_n^k}\sum_{s=0}^{S-1}\sum_{m=0}^{M-1}v_{n,s,m}^{k}\bigg\|^2\bigg| \theta^k\bigg]\\
    =f(\theta^k) -\delta \bigg\langle \nabla f(\theta^k), \frac{1}{N}\!\!\sum_{n=1}^{N}\sum_{s=0}^{S-1}\sum_{m=0}^{M-1}\!\!\EX\bigg[\nabla f_n(w_{n,s,m}^k)\bigg|\theta^k\bigg]\! \bigg\rangle \\
    +\frac{\delta^2 \mathsf{L}}{2N^2}\EX\bigg[\bigg\|\sum_{n=1}^{N}\frac{\vect{1}_n^k}{p_n^k}\sum_{s=0}^{S-1}\sum_{m=0}^{M-1}v_{n,s,m}^{k}\bigg\|^2\bigg| \theta^k\bigg]\\
    =f(\theta^k) -\delta\!\sum_{s=0}^{S-1}\!\sum_{m=0}^{M-1}\EX\bigg[\!\bigg\langle \!\nabla f(\theta^k), \frac{1}{N}\!\sum_{n=1}^{N}\!\nabla f_n(w_{n,s,m}^k)\bigg|\theta^k\bigg\rangle\bigg] \\
    +\frac{\delta^2 \mathsf{L}}{2N^2}\EX\bigg[\bigg\|\sum_{n=1}^{N}\frac{\vect{1}_n^k}{p_n^k}\sum_{s=0}^{S-1}\sum_{m=0}^{M-1}v_{n,s,m}^{k}\bigg\|^2\bigg| \theta^k\bigg]    
\end{multline}
In the second equality, we use the fact that $\EX[\vect{1}_n^k | \theta^k] = \EX[\vect{1}_n^k] = p_n^k$ and $\EX[v_{n,s,m}^{k}|\theta^k]=\EX[\nabla f_n(w_{n,s,m}^k)|\theta^k]$

By using \eqref{eq::upperbound} and plugging back into \eqref{eq::eqq}, we have:
\begin{multline}\label{eq::expect}
\EX[f(\theta^{k+1})|\theta^k] \leq f(\theta^k) 
-\delta\sum_{s=0}^{S-1}\sum_{m=0}^{M-1}\!\!\EX\bigg[\bigg\langle \nabla f(\theta^k), \frac{1}{N}\sum_{n=1}^{N}\nabla f_n(w_{n,s,m}^k)\bigg|\theta^k\bigg\rangle\bigg] \\
    +\frac{\mathsf{L}SM\delta^2}{2N}\sum_{n=1}^{N}\frac{1}{p_n^k}\sum_{s=0}^{S-1}\sum_{m=0}^{M-1}\EX\bigg[\bigg\|v_{n,s,m}^{k}\bigg\|^2\bigg| \theta^k\bigg]  
\end{multline}
Applying total expectation on both sides of \eqref{eq::expect} we have:
\begin{multline}
    \EX[f(\theta^{k+1})] \leq \EX[f(\theta^k)] -\delta\sum_{s=0}^{S-1}\sum_{m=0}^{M-1}\EX\bigg[\bigg\langle \nabla f(\theta^k), \frac{1}{N}\sum_{n=1}^{N}\nabla f_n(w_{n,s,m}^k)\bigg\rangle\bigg] \\
    +\frac{\mathsf{L}SM\delta^2}{2N}\sum_{n=1}^{N}\frac{1}{p_n^k}\sum_{s=0}^{S-1}\sum_{m=0}^{M-1}\EX\bigg[\bigg\|v_{n,s,m}^{k}\bigg\|^2\bigg].  
    \label{eq::11}
\end{multline}
By using \eqref{eq::12}, it follows from~\eqref{eq::11} that
\begin{multline}
\EX[f(\theta^{k+1})] \leq \EX[f(\theta^k)] +\frac{\delta^3 \mathsf{L}^2 (M-1)}{2N} \Bigg[ \sum_{n=1}^{N}\sum_{s=0}^{S-1}\sum_{m=0}^{M-1}\sum_{m^\prime=0}^{m-1}\EX\bigg[\bigg\| v_{n,s,m^\prime}^k \bigg\|^2\bigg]\Bigg] -\frac{\delta S M}{2}\EX\bigg[\bigg\|\nabla f(\theta^k)\bigg\|^2\bigg] \\
    +\frac{\mathsf{L}SM\delta^2}{2N}\sum_{n=1}^{N}\frac{1}{p_n^k}\sum_{s=0}^{S-1}\sum_{m=0}^{M-1}\EX\bigg[\bigg\|v_{n,s,m}^{k}\bigg\|^2\bigg] 
\end{multline}
Dividing both sides of equation over $\frac{1}{K}$. Then summing both sides from $k=0$ to $k=K-1$. We have the following:
\begin{multline}
\frac{1}{K}\sum_{k=0}^{K-1}\EX\bigg[ \bigg \|\nabla f(\theta^k) \bigg\|^2 \bigg]\leq \!\!\frac{2}{\delta K S M}(\EX[f(\theta^0] -\EX[f(\theta^k]))
+\frac{\delta^2 \mathsf{L}^2 (M-1)}{K S M N} \Bigg[ \sum_{k=0}^{K-1}\sum_{n=1}^{N}\sum_{s=0}^{S-1}\sum_{m=0}^{M-1} \sum_{m^\prime=0}^{m-1}\EX\bigg[\bigg\| v_{n,s,m^\prime}^k \bigg\|^2\bigg]\\
+\frac{\delta \mathsf{L}}{KN}\sum_{k=0}^{K-1}\sum_{n=1}^{N}\frac{1}{p_n^k}\sum_{s=0}^{S-1}\sum_{m=0}^{M-1}\EX\bigg[\bigg\|v_{n,s,m}^{k}\bigg\|^2\bigg] \Bigg]
\end{multline}
\begin{multline}
\frac{1}{K}\sum_{k=0}^{K-1}\EX\bigg[ \bigg \|\nabla f(\theta^k) \bigg\|^2 \bigg]\leq \frac{2}{\delta K S M}(f(\theta^0) -f^\star)
+\frac{\delta^2 \mathsf{L}^2 (M-1)}{K S M N} \Bigg[ \sum_{k=0}^{K-1}\sum_{n=1}^{N}\sum_{s=0}^{S-1}\sum_{m=0}^{M-1}\sum_{m^\prime=0}^{m-1}\EX\bigg[\bigg\| v_{n,s,m^\prime}^k \bigg\|^2\bigg] \Bigg] \\
+\frac{\delta \mathsf{L}}{KN}\sum_{k=0}^{K-1}\sum_{n=1}^{N}\frac{1}{p_n^k}\sum_{s=0}^{S-1}\sum_{m=0}^{M-1}\EX\bigg[\bigg\|v_{n,s,m}^{k}\bigg\|^2\bigg] 
\end{multline}
Which concludes the proof.
\end{proof}

\newcounter{mycounter}
\renewcommand{\themycounter}{A.\arabic{mycounter}}
\newtheorem{thmapp}[mycounter]{Theorem}

\end{document}